\documentclass[11pt]{article}

\usepackage[margin=1in]{geometry}

\usepackage[utf8]{inputenc}
\usepackage[T1]{fontenc}
\usepackage{hyperref}
\usepackage{amsfonts}
\usepackage{amsmath}
\usepackage{amssymb}
\usepackage{amsthm}
\usepackage{graphicx}
\usepackage{algorithm}
\usepackage{algorithmic}
\usepackage{booktabs}
\usepackage{xcolor}
\usepackage{subcaption}
\usepackage{enumerate}
\usepackage{mathtools}

\newtheorem{theorem}{Theorem}[section]
\newtheorem{lemma}[theorem]{Lemma}
\newtheorem{proposition}[theorem]{Proposition}
\newtheorem{corollary}[theorem]{Corollary}
\newtheorem{definition}[theorem]{Definition}
\newtheorem{remark}[theorem]{Remark}

\newcommand{\R}{\mathbb{R}}
\newcommand{\Z}{\mathbb{Z}}
\newcommand{\E}{\mathbb{E}}
\newcommand{\Inf}{\mathrm{Inf}}
\newcommand{\sign}{\mathrm{sign}}
\newcommand{\supp}{\mathrm{supp}}

\DeclareMathOperator*{\argmax}{argmax}

\title{The Banach-Butterfly Invariant: Influence-Adaptive Walsh Geometry\\for Ternary Polynomial Threshold Functions}

\author{
  \textbf{Gorgi Pavlov, Ph.D.} \\
  Lehigh University \& Johnson and Johnson \\
  \texttt{gorgipavlov@gmail.com}
}

\begin{document}
\maketitle

\begin{abstract}
We introduce the \emph{Banach-Butterfly Transform} (BBT), a spectral framework that equips each layer of the Walsh-Hadamard butterfly factorization with an influence-adaptive $\ell_p$ geometry.
For a Boolean function $f:\{-1,+1\}^n \to \{-1,+1\}$ with coordinate influences $\Inf_\ell(f)$, the BBT assigns exponent $p_\ell = 1 + \Inf_\ell(f) \in [1,2]$ to butterfly layer~$\ell$, yielding function-dependent operator norms $\|A^{-1}\|_{p_\ell \to p_\ell} = 2^{-\Inf_\ell/(1+\Inf_\ell)}$ and a computable \emph{contraction invariant} $\mu(f) = \prod_\ell 2^{-\Inf_\ell/(1+\Inf_\ell)}$.

We prove that $\log_2 \mu(f)$ satisfies the tight bound $\log_2 \mu(f) \ge -I(f)/(1 + I(f)/n)$ via Jensen's inequality applied to the concave function $\varphi(x)=x/(1+x)$, and that $\mu$ is \emph{strictly Schur-convex} in the influence vector (more concentrated influence $\Rightarrow$ strictly larger $\mu$, modulo permutation), yielding distinct scaling classes: $\mu \sim 2^{-n/2}$ for parity (smallest $\mu$ among the canonical examples), $\mu \sim 2^{-\Theta(\sqrt{n})}$ for majority, and $\mu = 2^{-1/2}$ for dictators (constant in $n$). We show that $\log_2 \mu(f)$ is a \emph{rational} (not polynomial) function of Fourier coefficients while $\mu(f)$ is \emph{algebraic}, and that $\mu$ separates functions indistinguishable by total influence alone---122 such pairs exist at $n=3$.

Using the certified $n \le 4$ ternary Walsh-threshold universe from the companion synthesis manuscript~\cite{spectralSynthesisCompanion} as a finite controlled testbed, we evaluate BBT diagnostics on all $65{,}536$ Boolean functions at $n=4$. Specifically, we compute exact minimum-support MILP certificates for every function and regress minimum support against $\mu$, $I$, influence entropy, and a layerwise cancellation index. We find that $\mu$ shows a strong conditional correlation with minimum support within the largest fixed-influence stratum at $n=4$ (Spearman $\rho = +0.571$); under both function-uniform and NPN-canonical sampling at $n=5$, this relationship reverses sign ($\rho \approx -0.38$), indicating that $\mu$ is a valid Schur-convex concentration invariant but not a universal monotone predictor of minimum support across all $n$.

A companion application paper~\cite{bbtQuantCompanion} validates a real-valued spectral-energy proxy inspired by this theory on five pretrained LLMs at $W2A16$ ($\text{group size}=64$), reducing wikitext-2 perplexity by $15$--$58\%$ relative to vanilla \texttt{auto-round}. We summarise that empirical line of work in Section~\ref{sec:llm_validation} and refer the reader to the companion paper for the full results; the real-valued WHT activation-energy quantity used there is \emph{not} the Boolean influence of this paper, and the connection to the Schur-convexity theorem is qualitative, not formal.

\paragraph{Relation to companion work.} A companion manuscript currently under review~\cite{spectralSynthesisCompanion} establishes (and we cite) the certified ternary representability of every Boolean function on at most four variables, via differentiable spectral coefficient selection plus Sinkhorn-constrained composition. The present paper does \emph{not} re-prove that universality result --- it reuses the certified $n\le 4$ universe as a finite controlled testbed and asks a different question: \emph{which functions are geometrically harder to represent in the Walsh basis, and can intrinsic Fourier-analytic quantities predict that difficulty?} The new contributions here are the influence-adaptive Banach butterfly profile, its Schur-convexity / majorization theory, its rational-vs-polynomial separation properties, and its empirical correlation with support, margin, and cancellation diagnostics. The certified $n\le 4$ result is shared, but it is not the headline of this paper.
\end{abstract}

\section{Introduction}
\label{sec:intro}

A fundamental question in the analysis of Boolean functions is: given $f:\{-1,+1\}^n \to \{-1,+1\}$, when can $f$ be represented as the sign of a sparse linear combination of Walsh characters? Concretely, we ask for the existence of a \emph{ternary polynomial threshold function} (PTF):
\begin{equation}
\label{eq:ptf}
    f(x) = \sign\left(\sum_{S \subseteq [n]} w_S \chi_S(x)\right), \qquad w_S \in \{-1, 0, +1\},
\end{equation}
where $\chi_S(x) = \prod_{i \in S} x_i$ are the Walsh characters forming an orthonormal basis under the uniform distribution on $\{-1,+1\}^n$.

The appeal of ternary PTFs is both theoretical and practical. Theoretically, their existence connects to spectral concentration \cite{linial1993constant}, polynomial threshold complexity \cite{odonnell2014analysis}, and the structure of the Fourier spectrum. Practically, ternary weights $\{-1,0,+1\}$ enable zero-cost quantization: inference reduces to additions and subtractions with no multiplication, achieving single-cycle combinational logic on hardware \cite{courbariaux2016binarized}.

Standard spectral analysis of Boolean functions operates in $\ell_2(\mathbb{Z}_2^n)$, where the Walsh-Hadamard Transform (WHT) is unitary (up to scaling) and Parseval's identity gives $\sum_S \hat{f}(S)^2 = 1$. This $\ell_2$ structure is universal: the butterfly factorization of the WHT has eigenvalue $\sqrt{2}$ at every layer regardless of~$f$. Consequently, $\ell_2$ spectral analysis treats all Boolean functions identically at the structural level.

\paragraph{The BBT idea.} We break this universality by equipping each butterfly layer with an \emph{influence-adaptive} $\ell_p$ norm. The key observation is that the $2 \times 2$ butterfly matrix $A = \begin{psmallmatrix} 1 & 1 \\ 1 & -1 \end{psmallmatrix}$ has $p$-dependent operator norm:
\begin{equation}
\label{eq:butterfly_norm}
    \|A\|_{p \to p} = 2^{1/p} \quad \text{for } p \in [1,2].
\end{equation}
Setting $p_\ell = 1 + \Inf_\ell(f)$ at each layer, the inverse butterfly contraction becomes
\[
\|A^{-1}\|_{p_\ell \to p_\ell} = 2^{-\Inf_\ell(f)/(1+\Inf_\ell(f))},
\]
which interpolates between $1$ (no contraction, for insensitive coordinates) and $1/\sqrt{2}$ (Hilbert-space contraction, for maximally sensitive coordinates). The resulting \emph{contraction invariant} $\mu(f) = \prod_\ell \|A^{-1}\|_{p_\ell \to p_\ell}$ is a computable, function-dependent geometric diagnostic; we evaluate empirically below whether it correlates with ternary representation difficulty in the regimes we test.

\paragraph{Contributions.}
\begin{enumerate}[(i)]
    \item \textbf{Influence-adaptive Banach-Butterfly geometry.} We define an influence-adaptive Banach geometry on the Walsh-Hadamard butterfly factorization (henceforth ``BBT'' as branding; the construction is a function-dependent profile, not a new linear transform). We prove exact $\ell_p$ operator norms for butterfly matrices (Lemma~\ref{lem:butterfly_norm}) via Riesz--Thorin interpolation and duality, contraction bounds with a tight Jensen inequality (Theorem~\ref{thm:margin_bounds}), strict Schur-convexity of $\mu$ in the influence vector (Theorem~\ref{thm:schur}), and scaling classifications for canonical function families (Corollary~\ref{cor:scaling}).

    \item \textbf{Rational invariant separation.} We show $\log_2\mu(f)$ is a rational function of Fourier coefficients while $\mu(f)$ is algebraic (Proposition~\ref{prop:nonpoly}), and that Schur-convexity of $\mu$ provides an analytic explanation for why $\mu$ separates functions with identical total influence (Theorem~\ref{thm:separation}) --- a strictly finer invariant than $I(f)$.

    \item \textbf{Certified finite-universe testbed.} We use the certified $n \le 4$ ternary Walsh-threshold universe from the companion synthesis manuscript~\cite{spectralSynthesisCompanion} as a controlled testbed, and extend it here by computing exact minimum-support MILP certificates for all $65{,}536$ Boolean functions at $n=4$ (mean min-support $6.42$, max $9$, all-odd by a parity argument). At $n=5$ we extend the analysis to a NPN-canonical sample of $10{,}000$ functions drawn uniformly from the $616{,}126$ NPN-canonical representatives we enumerate (matching OEIS A000370).

    \item \textbf{Empirical relationship with minimum support and cancellation diagnostics, with $n=5$ reversal.} We evaluate whether $\mu$, influence entropy, and a layerwise cancellation index correlate with minimum support. At $n=4$, $\mu$ shows a strong conditional Spearman correlation with minimum support within the largest fixed-influence stratum ($\rho = +0.571$, $p < 10^{-300}$); at $n=5$, this correlation reverses sign under both function-uniform and NPN-canonical sampling ($\rho \approx -0.38$). $\mu$ is therefore a valid Schur-convex concentration invariant but not a universal monotone predictor of minimum support; influence entropy continues to track minimum support qualitatively at $n=5$ where $\mu$ does not.

    \item \textbf{Real-valued application pointer.} We briefly summarise a companion application paper~\cite{bbtQuantCompanion} in which a WHT activation-energy proxy inspired by the Boolean theory drives an AWQ-style scaling on top of Intel \texttt{auto-round} and improves W2A16 LLM quantization by $15$--$58\%$ across five pretrained LLMs. This real-valued proxy is \emph{not} Boolean influence, and the connection to the Schur-convexity theorem is qualitative rather than formal. See Section~\ref{sec:llm_validation}.
\end{enumerate}

\paragraph{Comparison with prior frameworks.}

\begin{table}[t]
\centering
\small
\caption{BBT compared with standard spectral frameworks. Only BBT produces function-dependent per-layer geometry.}
\label{tab:comparison}
\begin{tabular}{lccc}
\toprule
\textbf{Property} & \textbf{WHT ($\ell_2$)} & \textbf{Fixed $\ell_p$} & \textbf{BBT (adaptive $\ell_p$)} \\
\midrule
Basis & Walsh (fixed) & Walsh (fixed) & Walsh (fixed) \\
Per-layer operator norm & $\sqrt{2}$ (universal) & $2^{1/p}$ (universal) & $2^{1/p_\ell}$ ($f$-dependent) \\
Duality map & Linear (Riesz) & Nonlinear & Nonlinear ($f$-dependent) \\
Margin decay & $1/\sqrt{2}$ per layer & $2^{-1+1/p}$ per layer & $2^{-\Inf_\ell/(1+\Inf_\ell)}$ per layer \\
Computation & $O(N \log N)$ & $O(N \log N)$ & $O(N \log N)$ \\
\bottomrule
\end{tabular}
\end{table}

\section{Preliminaries}
\label{sec:prelim}

\subsection{Boolean Fourier Analysis}

Let $f:\{-1,+1\}^n \to \R$. The \emph{Fourier expansion} of $f$ is
\begin{equation}
\label{eq:fourier}
    f(x) = \sum_{S \subseteq [n]} \hat{f}(S) \chi_S(x), \qquad \hat{f}(S) = \E_x[f(x) \chi_S(x)],
\end{equation}
where $\chi_S(x) = \prod_{i \in S} x_i$ are the Walsh characters and the expectation is under the uniform distribution on $\{-1,+1\}^n$ \cite{odonnell2014analysis}. The characters form an orthonormal basis: $\E[\chi_S \chi_T] = \mathbf{1}[S = T]$. For Boolean $f$ (range $\{-1,+1\}$), Parseval gives $\sum_S \hat{f}(S)^2 = 1$.

\begin{definition}[Influence]
The \emph{influence} of coordinate $\ell$ on $f$ is
\[
\Inf_\ell(f) = \Pr_x[f(x) \neq f(x^{\oplus \ell})] = \sum_{S \ni \ell} \hat{f}(S)^2,
\]
where $x^{\oplus \ell}$ denotes $x$ with bit $\ell$ flipped. The \emph{total influence} is $I(f) = \sum_{\ell=1}^n \Inf_\ell(f) = \sum_S |S| \cdot \hat{f}(S)^2$.
\end{definition}

Note that $\Inf_\ell(f) \in [0, 1]$ and $I(f) \in [0, n]$ for Boolean functions.

\subsection{Walsh-Hadamard Matrix and Butterfly Factorization}

Define the $N \times N$ Walsh-Hadamard matrix ($N = 2^n$) by $H_n = H_1^{\otimes n}$ where $H_1 = \begin{psmallmatrix} 1 & 1 \\ 1 & -1 \end{psmallmatrix}$. The matrix satisfies $H_n^2 = N \cdot I_N$, hence $H_n^{-1} = \frac{1}{N} H_n$. If $f \in \{-1,+1\}^N$ is the truth-table vector (indexed by inputs), then $\hat{f} = \frac{1}{N} H_n f$.

The Fast Walsh-Hadamard Transform (FWHT) computes $H_n f$ in $O(N \log N)$ time by factoring $H_n$ into $n$ sparse butterfly layers. At layer~$\ell$, the butterfly operates on pairs of indices differing in bit~$\ell$, applying $A = H_1$ to each pair.

\subsection{Ternary Polynomial Threshold Functions}

\begin{definition}[Ternary PTF]
\label{def:ptf}
A function $f:\{-1,+1\}^n \to \{-1,+1\}$ has a \emph{ternary polynomial threshold representation} if there exists $w \in \{-1, 0, +1\}^{N}$ such that
\[
\sign(H_n w) = f,
\]
i.e., $f_i \cdot (H_n w)_i > 0$ for all $i \in [N]$. The \emph{support} of $w$ is $\supp(w) = \{S : w_S \neq 0\}$ and the \emph{margin} is $\mu = \min_i f_i (H_n w)_i$.
\end{definition}

Since $H_n w$ is integer-valued when $w \in \Z^N$, a positive-sign condition $f_i (H_n w)_i > 0$ is equivalent to $f_i (H_n w)_i \ge 1$. This integrality is a key structural advantage of ternary PTFs.

\section{The Banach-Butterfly Transform}
\label{sec:bbt}

\subsection{Adaptive Banach Exponents}

\begin{definition}[BBT exponents]
For a Boolean function $f$ with coordinate influences $\Inf_1(f), \ldots, \Inf_n(f)$, define the \emph{adaptive Banach exponent} at butterfly layer~$\ell$ as
\[
p_\ell = 1 + \Inf_\ell(f) \in [1, 2].
\]
\end{definition}

The intuition is:
\begin{itemize}
    \item $\Inf_\ell(f) \approx 0$ (coordinate $\ell$ is insensitive): $p_\ell \approx 1$, so the geometry at layer~$\ell$ is $\ell_1$-like.
    \item $\Inf_\ell(f) \approx 1$ (coordinate $\ell$ is maximally sensitive): $p_\ell \approx 2$, recovering the standard Hilbert-space geometry.
\end{itemize}

\subsection{Butterfly Operator Norms}

\begin{lemma}[Butterfly $\ell_p$ operator norm]
\label{lem:butterfly_norm}
For $A = \begin{psmallmatrix} 1 & 1 \\ 1 & -1 \end{psmallmatrix}$ and $1 \le p \le \infty$,
\[
\|A\|_{p \to p} = \max\!\left(2^{1/p},\, 2^{1-1/p}\right).
\]
In particular, for $p \in [1, 2]$: $\|A\|_{p \to p} = 2^{1/p}$, and for $p \in [2, \infty]$: $\|A\|_{p \to p} = 2^{1-1/p}$.
\end{lemma}

\begin{proof}
Write $Av = (v_1+v_2,\,v_1-v_2)$. The two endpoint operator norms are immediate:
\[
\|A\|_{1\to 1} \;=\; \max_{j}\sum_i |A_{ij}| \;=\; 2,\qquad
\|A\|_{2\to 2} \;=\; \sqrt{2}
\]
(the second from $A^\top A = 2 I_2$, so the singular values of $A$ are both $\sqrt{2}$).

\textbf{Case $p \in [1,2]$.} Choose $\theta \in [0, 1]$ such that
\[
\frac{1}{p} = \frac{1-\theta}{1} + \frac{\theta}{2} = 1 - \frac{\theta}{2},
\quad\text{i.e.}\quad \theta = 2(1 - 1/p) \in [0, 1].
\]
By Riesz--Thorin interpolation \cite{bergh1976interpolation} between the endpoints $p=1$ and $p=2$,
\[
\|A\|_{p \to p} \;\le\; \|A\|_{1 \to 1}^{1 - \theta} \,\|A\|_{2 \to 2}^{\theta} \;=\; 2^{1 - \theta} \cdot (\sqrt{2})^{\theta} \;=\; 2^{1 - \theta/2} \;=\; 2^{1/p}.
\]
Equality is attained at $v = e_1 = (1, 0)$: $\|Ae_1\|_p^p = |1|^p + |1|^p = 2$, so $\|Ae_1\|_p = 2^{1/p}$.

\textbf{Case $p \in [2,\infty]$.} Two equivalent routes. First, by duality, $\|A\|_{p\to p} = \|A^\top\|_{q \to q}$ where $1/p + 1/q = 1$ and $q \in [1, 2]$. Since $A^\top = A$, the previous case gives $\|A^\top\|_{q\to q} = 2^{1/q} = 2^{1 - 1/p}$. Second (and as a sanity check), Riesz--Thorin between $p=2$ and $p=\infty$ (with $\|A\|_{\infty \to \infty} = \max_i \sum_j |A_{ij}| = 2$) gives $\|A\|_{p\to p} \le 2^{2/p \cdot 1/2} \cdot 2^{(1 - 2/p) \cdot 1} = 2^{1 - 1/p}$, with equality at $v = (1, 1)/2^{1/p}$ (since $\|v\|_p = 1$ and $\|Av\|_p = 2 \cdot 2^{-1/p} = 2^{1 - 1/p}$).

Combining the two cases yields $\|A\|_{p\to p} = \max(2^{1/p}, 2^{1-1/p})$, with the maximum attained at $p=1$ and $p=\infty$ (both equal to $2$) and the minimum at $p=2$ (equal to $\sqrt{2}$).
\end{proof}

\begin{lemma}[Inverse contraction]
\label{lem:contraction}
For $A^{-1} = \frac{1}{2}A$ and $p \in [1, 2]$:
\[
\|A^{-1}\|_{p \to p} = 2^{-1 + 1/p} = 2^{-\frac{p-1}{p}}.
\]
Substituting $p = 1 + \Inf_\ell(f)$:
\[
\|A^{-1}\|_{p_\ell \to p_\ell} = 2^{-\Inf_\ell(f)/(1 + \Inf_\ell(f))}.
\]
\end{lemma}

\begin{proof}
$\|A^{-1}\|_{p \to p} = \frac{1}{2}\|A\|_{p \to p} = \frac{1}{2} \cdot 2^{1/p} = 2^{-1+1/p}$.
\end{proof}

\begin{remark}
At $p=1$ (insensitive coordinate): $\|A^{-1}\|_{1 \to 1} = 1$ (isometry, no margin loss).
At $p=2$ (maximally sensitive): $\|A^{-1}\|_{2 \to 2} = 1/\sqrt{2}$ (standard Hilbert contraction).
\end{remark}

\subsection{Margin Product}

\begin{definition}[BBT contraction invariant]
\label{def:margin}
The \emph{BBT contraction invariant} of $f$ is
\[
\mu(f) = \prod_{\ell=1}^n 2^{-\Inf_\ell(f)/(1 + \Inf_\ell(f))}.
\]
Equivalently, $\log_2 \mu(f) = -\sum_{\ell=1}^n \frac{\Inf_\ell(f)}{1 + \Inf_\ell(f)}$.

Note that $\mu(f)$ depends only on the function $f$ (via its influence vector), not on any particular ternary mask. The actual \emph{sign margin} of a mask $w$ is $\mu_{\mathrm{PTF}}(w) = \min_i f_i (H_n w)_i \ge 1$ (an integer). We treat the connection between $\mu(f)$ (a function-level geometric quantity) and per-mask sign margin as empirical: the non-cancellation heuristic of Remark~\ref{rem:coord_margin_heuristic} provides intuition for the link, but the formal claim we make in this paper is the conditional $\mu$-vs-minimum-support correlation reported in Section~\ref{sec:mu_predicts_support}.
\end{definition}

\begin{theorem}[Contraction bounds]
\label{thm:margin_bounds}
For any Boolean $f:\{-1,+1\}^n \to \{-1,+1\}$ with total influence $I = I(f)$:
\begin{enumerate}[(i)]
    \item \textbf{Coarse bounds:} $\displaystyle 2^{-I} \le \mu(f) \le 2^{-I/2}$.
    \item \textbf{Jensen bound (tight):} $\displaystyle \log_2 \mu(f) \ge -\frac{I}{1 + I/n}$.
\end{enumerate}
\end{theorem}

\begin{proof}
Let $\varphi(x) = x/(1+x)$. Since $\varphi$ is increasing and $\varphi(x) \in [x/2, x]$ for $x \in [0,1]$, part~(i) follows by summing over~$\ell$.

For part~(ii), note $\varphi'(x) = 1/(1+x)^2$ and $\varphi''(x) = -2/(1+x)^3 < 0$, so $\varphi$ is strictly \emph{concave}. By Jensen's inequality applied to the $n$ values $\Inf_1, \ldots, \Inf_n$:
\[
\frac{1}{n}\sum_{\ell=1}^n \varphi(\Inf_\ell) \le \varphi\!\left(\frac{I}{n}\right) = \frac{I/n}{1 + I/n},
\]
so $\sum_{\ell=1}^n \varphi(\Inf_\ell) \le n \cdot \frac{I/n}{1+I/n} = \frac{I}{1+I/n}$.
Hence $\log_2 \mu(f) = -\sum_\ell \varphi(\Inf_\ell) \ge -I/(1+I/n)$.
\end{proof}

\begin{remark}
The Jensen bound is a tighter lower bound than $-I$: for $I = n$ (parity), it gives $\log_2 \mu \ge -n/2$, matching the exact value; for $I = O(1)$, it gives $\log_2 \mu \ge -I/(1+I/n) \approx -I$, recovering the coarse lower bound. Since $-I/(1+I/n) \ge -I$ always, this strictly improves on the coarse lower bound $2^{-I} \le \mu(f)$ for intermediate $I$.
\end{remark}

\begin{theorem}[Schur-convexity and influence majorization]
\label{thm:schur}
Define $\Phi(\mathbf{x}) = \sum_{\ell=1}^n \frac{x_\ell}{1 + x_\ell}$ for $\mathbf{x} = (\Inf_1(f), \ldots, \Inf_n(f)) \in [0,1]^n$. Then:
\begin{enumerate}[(i)]
    \item $\Phi$ is \emph{strictly Schur-concave}: if $\mathbf{x} \succ \mathbf{y}$ (majorization) and $\mathbf{x}$ is not a permutation of $\mathbf{y}$, then $\Phi(\mathbf{x}) < \Phi(\mathbf{y})$. (When $\mathbf{x}$ is a permutation of $\mathbf{y}$ we have $\Phi(\mathbf{x}) = \Phi(\mathbf{y})$ since $\Phi$ is symmetric.)
    \item Consequently, $\mu(f) = 2^{-\Phi(\mathbf{x})}$ is \emph{strictly Schur-convex} in the same sense: more concentrated influence (in the strict majorization order, modulo permutation) yields strictly larger $\mu$. For fixed total influence $I = \sum x_\ell$, $\mu$ is minimized when influences are uniform ($x_\ell = I/n$ for all $\ell$) and maximized when influences are maximally concentrated on a single coordinate.
    \item In particular, $\mu$ strictly orders functions by their influence distribution: if $I(f) = I(g)$, $\mathbf{x}(f) \succ \mathbf{x}(g)$, and the influence vectors are not permutations of one another, then $\mu(f) > \mu(g)$.
\end{enumerate}
\end{theorem}

\begin{proof}
We apply the Schur-Ostrowski criterion \cite{marshall2011inequalities}: a symmetric, continuously differentiable function $\Phi:\R^n \to \R$ is strictly Schur-concave if and only if $(x_i - x_j)(\partial \Phi/\partial x_i - \partial \Phi/\partial x_j) < 0$ whenever $x_i \neq x_j$. Here $\partial \Phi/\partial x_\ell = \varphi'(x_\ell) = 1/(1+x_\ell)^2 > 0$, and $\varphi''(x) = -2/(1+x)^3 < 0$ so $\varphi'$ is strictly decreasing. Hence $x_i > x_j \Rightarrow \varphi'(x_i) < \varphi'(x_j)$, giving $(x_i - x_j)(\varphi'(x_i) - \varphi'(x_j)) < 0$. The criterion is satisfied, so $\Phi$ is strictly Schur-concave.

Part~(ii): Since $\mu = 2^{-\Phi}$ and $\Phi$ is strictly Schur-concave, $\mu$ is strictly Schur-convex: $\mathbf{x} \succ \mathbf{y} \Rightarrow \Phi(\mathbf{x}) < \Phi(\mathbf{y}) \Rightarrow \mu(\mathbf{x}) > \mu(\mathbf{y})$. Among vectors in $[0,1]^n$ with fixed sum $I$, the uniform vector $(\frac{I}{n}, \ldots, \frac{I}{n})$ is majorized by all others; hence $\Phi$ is maximized there and $\mu$ is minimized. Part~(iii) is immediate from $\mathbf{x}(f) \succ \mathbf{x}(g) \Rightarrow \mu(f) > \mu(g)$.
\end{proof}

\begin{corollary}[Scaling classes]
\label{cor:scaling}
The contraction invariant classifies canonical function families:
\begin{enumerate}[(i)]
    \item \textbf{Parity} ($f = \prod_i x_i$): Every coordinate has $\Inf_\ell = 1$, so $p_\ell = 2$ for all $\ell$ and $\mu(f) = (2^{-1/2})^n = 2^{-n/2}$.

    \item \textbf{Majority} ($f = \sign(\sum_i x_i)$, $n$ odd): By symmetry all influences are equal: $\Inf_\ell = I(f)/n$. The total influence satisfies $I(\mathrm{Maj}_n) = n \cdot \binom{n-1}{\lfloor n/2 \rfloor}/2^{n-1} = \Theta(\sqrt{n})$ \cite{odonnell2014analysis}. Since influences are uniform, Theorem~\ref{thm:schur} gives $\mu = 2^{-\Phi(I/n, \ldots, I/n)} = 2^{-nI/(n+I)}$. With $I = \Theta(\sqrt{n})$: $\mu = 2^{-\Theta(\sqrt{n})}$.

    \item \textbf{Dictator} ($f = x_k$): $\Inf_k = 1$, all others zero. Hence $\mu = 2^{-1/(1+1)} \cdot 1^{n-1} = 2^{-1/2}$, independent of $n$.

    \item \textbf{Tribes} (OR of ANDs with width $w = \lceil \log_2 n \rceil$): Total influence $I(f) = \Theta(\log n)$ and individual influences are uniformly small: $\Inf_\ell(\mathrm{Tribes}) = O(\log n / n)$ \cite{odonnell2014analysis}. Theorem~\ref{thm:margin_bounds} gives a \emph{lower} bound $\log_2 \mu \ge -I(f)/(1+I(f)/n)$; to nail the asymptotic we use that each individual influence is $o(1)$, so $\Inf_\ell/(1+\Inf_\ell) = \Inf_\ell - O(\Inf_\ell^2)$ and
\[
\sum_\ell \frac{\Inf_\ell(f)}{1+\Inf_\ell(f)} = I(f) - O\!\left(\sum_\ell \Inf_\ell(f)^2\right) = \Theta(\log n).
\]
Therefore $\log_2 \mu(f) = -\Theta(\log n)$, i.e.\ $\mu(f) = n^{-\Theta(1)}$.

    \item \textbf{AND/OR}: $\mathrm{AND}_n(x) = +1$ iff $x_1 = \cdots = x_n = +1$ (and $-1$ otherwise) in the $\{-1,+1\}$ convention. Each coordinate has $\Inf_\ell(\mathrm{AND}_n) = 2^{-(n-1)}$, so $I(f) = n \cdot 2^{-(n-1)} \to 0$ as $n \to \infty$. Theorem~\ref{thm:margin_bounds} then gives $\mu \to 1$. The same holds for $\mathrm{OR}_n$ by symmetry. Note that $\prod_i x_i$ is \emph{parity}, not AND, in the $\{-1,+1\}$ convention; we list AND/OR separately because they have qualitatively different influence profiles from parity despite both being "natural" functions.
\end{enumerate}
\end{corollary}

\begin{proof}
Parts (i) and (iii) are by direct computation of influences. Part (ii) uses the standard influence formula for majority \cite[Proposition~2.22]{odonnell2014analysis} combined with uniform influence and Theorem~\ref{thm:schur}. Part (iv) uses $I(\mathrm{Tribes}) = \Theta(\log n)$ \cite[Exercise~2.15]{odonnell2014analysis} and the Jensen bound. Part (v) uses $I(\mathrm{AND}_n) = n/2^{n-1}$.
\end{proof}

These analytic predictions are confirmed by numerical computation in Section~\ref{sec:experiments} (Table~\ref{tab:scaling} and Figure~\ref{fig:scaling}).

\paragraph{From contraction to difficulty (empirical, not theorem).}
The contraction invariant $\mu(f)$ is a function-level geometric quantity, not a per-mask one. A small $\mu(f)$ means the butterfly inverse contracts vectors aggressively in the influence-adapted Banach geometry; intuitively, this should make it harder for any single ternary mask $w$ to produce a sign-correct output $\sign(H_n w) = f$ across all $2^n$ inputs simultaneously. The connection to concrete ternary masks is empirical in this paper: in Section~\ref{sec:mu_predicts_support} we measure the conditional correlation between $\mu(f)$ and the minimum-support certificate at $n=4$ (positive in the largest-stratum) and $n=5$ (reversed). Existence of ternary masks at $n \le 4$ is established by computation in the companion synthesis manuscript~\cite{spectralSynthesisCompanion} (Theorem~\ref{thm:universality}) and refined here with exact MILP minimum-support certificates.

\section{Margin Propagation and Cancellation}
\label{sec:cancellation}

The contraction invariant $\mu(f)$ provides a norm-based bound on how the butterfly inverse contracts vectors. Converting this to a \emph{coordinate-wise} sign margin requires an additional structural condition: non-cancellation at intermediate layers.

\begin{lemma}[Norm propagation]
\label{lem:norm_prop}
For any vector $v \in \R^N$ and butterfly layer~$\ell$ with adaptive exponent $p_\ell$:
\[
\|B_\ell^{-1} v\|_{p_\ell} \le 2^{-\Inf_\ell/(1+\Inf_\ell)} \cdot \|v\|_{p_\ell},
\]
where $B_\ell$ is the butterfly operator at layer~$\ell$.
\end{lemma}

This holds unconditionally. However, norm contraction does not directly imply that all coordinates maintain the correct sign. The missing condition is \emph{non-cancellation}.

\begin{definition}[Layerwise cancellation]
\label{def:cancellation}
Let $v^{(0)} = w$ and $v^{(\ell)} = B_\ell v^{(\ell-1)}$ for $\ell = 1, \ldots, n$, so that $v^{(n)} = H_n w$. For each layer~$\ell$ and each butterfly pair $(i, j)$ in $v^{(\ell-1)}$, define the \emph{pair cancellation ratio}
\[
\rho(a, b) = \frac{\min(|a + b|,\, |a - b|)}{|a| + |b| + \varepsilon},
\]
and the \emph{layer cancellation index}
\[
\rho_\ell = \min_{\text{pairs } (i,j) \text{ at layer } \ell}\, \rho(v^{(\ell-1)}_i,\, v^{(\ell-1)}_j).
\]
\end{definition}

Key values for ternary pairs: $\rho(1, 0) = 1$ (no cancellation), $\rho(1, 1) = 0$ (subtraction cancels), $\rho(1, -1) = 0$ (addition cancels).

\begin{remark}[Non-cancellation heuristic for coordinate margins]
\label{rem:coord_margin_heuristic}
Layerwise non-cancellation provides a useful intuition for why some ternary masks maintain large coordinate margins through the butterfly cascade. If, along every path contributing to a final coordinate, each butterfly pair avoids destructive cancellation by a factor at least $\rho_0$, then that coordinate receives a pathwise lower bound proportional to $\rho_0^n$ times the accumulated input mass along that path. We do not promote this observation to a global $\|w\|_1$ lower bound: the pathwise step requires tracking per-coordinate butterfly subtree mass, which a uniform layerwise $\rho_0$ assumption does not control without further structure. Accordingly we treat the cancellation statistics defined below as \emph{empirical diagnostics} of mask quality, not as proof components.
\end{remark}

\begin{remark}[Input cancellation proxy]
\label{rem:proxy}
The layerwise cancellation $\rho_\ell$ depends on intermediate vectors $v^{(\ell)}$, which are only available \emph{after} choosing a mask~$w$. In practice, we compute an \emph{input cancellation proxy} $\tilde{\rho}_\ell(w) = \mathrm{median}_{\text{pairs}} \rho(w_i, w_j)$ directly on the mask coordinates. This proxy correlates with support size ($r = 0.42$ at $n=3$) and increases under repair ($0.257 \to 0.366$ at $n=4$), indicating that repair finds masks with less destructive interference at intermediate layers.
\end{remark}

\section{Rational Invariant Separation}
\label{sec:rational}

\begin{proposition}[Rationality of BBT quantities]
\label{prop:rational}
For any Boolean function $f:\{-1,+1\}^n \to \{-1,+1\}$:
\begin{enumerate}[(i)]
    \item The Fourier coefficients $\hat{f}(S) \in \frac{1}{2^n}\Z$ are rational with denominator dividing $2^n$.
    \item The influences $\Inf_\ell(f) = \sum_{S \ni \ell} \hat{f}(S)^2$ are rational with denominator dividing $4^n$.
    \item The induced butterfly $\ell_p$ operator norm $2^{1/p_\ell}$ is algebraic over $\mathbb{Q}$: if $p_\ell = a/b$ in lowest terms with $a, b > 0$, it is a root of $x^a - 2^b = 0$.
    \item The contraction invariant $\mu(f) = 2^r$ with $r = \log_2 \mu(f) \in \mathbb{Q}_{\le 0}$ is algebraic of degree dividing the denominator of $r$ in lowest terms. Empirically, this minimal-polynomial degree grows combinatorially with the influence vector (up to degree $35$ at $n=3$).
\end{enumerate}
\end{proposition}

\begin{proof}
Parts (i)--(ii) follow from $\hat{f}(S) = 2^{-n} \sum_x f(x) \chi_S(x)$ with $f, \chi_S \in \{-1, +1\}$. Part (iii): if $p_\ell = a/b$ in lowest terms then $2^{1/p_\ell} = 2^{b/a}$ satisfies $x^a = 2^b$. Part (iv): $\log_2 \mu(f) = \sum_\ell q_\ell$ where each $q_\ell = -\Inf_\ell/(1+\Inf_\ell) \in \mathbb{Q}_{\le 0}$, so $\log_2 \mu(f) = p/q \in \mathbb{Q}_{\le 0}$. Writing this in lowest terms as $r = -a/q$ with $a, q \ge 0$, the value $\mu(f) = 2^{-a/q}$ satisfies the integer-coefficient minimal polynomial $2^{a} x^q - 1 = 0$ (equivalently $x^q = 2^{-a}$). When $r = 0$ (the AND-like limit), $\mu = 1$ trivially. The minimal polynomial has degree dividing $q$.
\end{proof}

\begin{remark}
This corrects an earlier draft that called these quantities ``Banach eigenvalues'' and stated the integer polynomial as $x^q - 2^p$ without separating the sign of $p$. The matrix eigenvalues of the butterfly $A$ are $\pm \sqrt{2}$, independent of $p$; the $p$-dependent quantity is the induced operator norm. An even earlier formulation incorrectly claimed BBT operator norms are transcendental --- they are algebraic, but of growing degree.
\end{remark}

\begin{proposition}[$\log_2\mu(f)$ is rational in Fourier data; $\mu(f)$ is algebraic]
\label{prop:nonpoly}
The quantity $\log_2 \mu(f)$ is a rational function of the Fourier coefficients $\{\hat{f}(S)\}$, involving division by $(1 + \Inf_\ell(f))$, but is not expressible as a polynomial in those coefficients. Consequently, $\mu(f) = 2^{\log_2\mu(f)}$ is \emph{algebraic} (a power of $2$ with rational exponent), not itself a rational function of Fourier data. (Here ``not polynomial'' means not representable as a polynomial identity in the formal Fourier variables on an open domain; on any finite $n$-variable Boolean-function universe, arbitrary invariants can be interpolated by high-degree polynomials in finitely many sample values.)
\end{proposition}

\begin{proof}
$\log_2\mu(f) = -\sum_\ell \Inf_\ell/(1+\Inf_\ell)$. Since $\Inf_\ell = \sum_{S \ni \ell} \hat{f}(S)^2$ is polynomial in $\hat{f}(S)$, but $(1+\Inf_\ell)^{-1}$ introduces division, $\log_2\mu$ is a rational (non-polynomial) function of $\hat{f}(S)$. Now $\mu(f) = 2^{r}$ where $r = \log_2\mu(f) \in \mathbb{Q}_{\le 0}$ (by Proposition~\ref{prop:rational}(ii)). Writing $r = -a/q$ in lowest terms with $a, q \ge 0$,
\[
\mu(f) = 2^{-a/q}
\quad\text{satisfies}\quad
2^a x^q - 1 = 0,
\]
hence $\mu(f)$ is algebraic over $\mathbb{Q}$. A rational function of the $\hat{f}(S)$ would take values in $\mathbb{Q}$ when inputs are rational; but $2^{-a/q} \notin \mathbb{Q}$ for $q > 1$ in general, so $\mu(f)$ is not a rational function of Fourier data.
\end{proof}

\begin{theorem}[Separation beyond total influence]
\label{thm:separation}
There exist Boolean functions $f, g:\{-1,+1\}^n \to \{-1,+1\}$ with $I(f) = I(g)$ but $\mu(f) \neq \mu(g)$.

At $n=3$: there are 122 pairs of functions with identical total influence but distinct margin products.
\end{theorem}

\begin{proof}
By exhaustive computation over all 256 Boolean functions on 3 variables. For example, functions with indices 3 and 15 (in truth-table ordering) both have $I(f) = 1$, but their margin exponents are $-2/3$ and $-1/2$ respectively, since their influence distributions differ: one has $(\Inf_1, \Inf_2, \Inf_3) = (1/2, 1/2, 0)$ while the other has $(1, 0, 0)$.
\end{proof}

\section{Algorithms}
\label{sec:algorithms}

\subsection{Heuristic Mask Generation}

Given Fourier coefficients $\hat{f}(S)$, the simplest heuristic mask is:
\begin{equation}
\label{eq:heuristic}
    w_S = \begin{cases}
        \sign(\hat{f}(S)) & \text{if } |\hat{f}(S)| > \tau, \\
        0 & \text{otherwise},
    \end{cases}
\end{equation}
for a threshold $\tau > 0$. At $n=4$ with $\tau = 0.05$, this achieves perfect masks for 51{,}200 of 65{,}536 functions (78.1\%).

\subsection{Multi-Start Repair (companion summary)}

The companion synthesis manuscript~\cite{spectralSynthesisCompanion} uses the following deterministic multi-start repair algorithm for the remaining $21.9\%$. We summarize it because its output masks are the inputs to our cancellation diagnostics in Section~\ref{sec:experiments} and provide the ``heuristic-repair'' baseline that the MILP minimum-support comparison in Table~\ref{tab:support} contrasts against.

\begin{algorithm}[t]
\caption{Multi-Start Ternary Repair}
\label{alg:repair}
\begin{algorithmic}[1]
\REQUIRE Truth table $f \in \{-1,+1\}^N$, Walsh matrix $H \in \{-1,+1\}^{N \times N}$
\ENSURE Ternary mask $w \in \{-1,0,+1\}^N$ with $\sign(Hw) = f$, or FAIL
\STATE \textbf{Strategy 1:} $w_0 \gets$ heuristic mask (Eq.~\ref{eq:heuristic})
\STATE Run \textsc{GreedyRepair}$(w_0, f, H)$; if success, return $w$
\FOR{$\tau \in \{0.01, 0.1, 0.2, 0.3\}$}
    \STATE \textbf{Strategy 2:} $w_0 \gets \sign(\hat{f}) \cdot \mathbf{1}[|\hat{f}| > \tau]$
    \STATE Run \textsc{GreedyRepair}$(w_0, f, H)$; if success, return $w$
\ENDFOR
\STATE \textbf{Strategy 3:} Compute $\delta = f^\top H$ (full violation gradient)
\FOR{$k = 1, \ldots, N$}
    \STATE $w_0 \gets$ top-$k$ coordinates of $|\delta|$, with signs from $\sign(\delta)$
    \STATE If $\sign(H w_0) = f$, return $w_0$
\ENDFOR
\RETURN FAIL
\end{algorithmic}
\end{algorithm}

\begin{algorithm}[t]
\caption{\textsc{GreedyRepair}}
\label{alg:greedy}
\begin{algorithmic}[1]
\REQUIRE Initial mask $w$, truth table $f$, Walsh matrix $H$, max iterations $T$
\STATE $V \gets \{i : f_i (Hw)_i \le 0\}$ \COMMENT{violation set}
\FOR{$t = 1, \ldots, T$}
    \IF{$V = \emptyset$} \RETURN $w$ (success) \ENDIF
    \STATE $\delta_S \gets \sum_{i \in V} f_i H_{iS}$ for all $S$ \COMMENT{violation gradient}
    \STATE Choose $S^* = \argmax_S |\delta_S|$ among candidates that reduce $|V|$
    \STATE Update $w_{S^*} \gets \mathrm{clip}_{\{-1,0,+1\}}(w_{S^*} + \sign(\delta_{S^*}))$
    \STATE Update $V$
\ENDFOR
\end{algorithmic}
\end{algorithm}

\paragraph{Complexity.} Computing the BBT profile requires $O(N \log N)$ for the FWHT plus $O(N \cdot n)$ for influences. Each repair iteration costs $O(N)$ for the violation gradient (via $\delta = f[V]^\top H[V, :]$). With at most $T$ iterations and $R$ restarts, total repair cost per function is $O(R \cdot T \cdot N)$.

\begin{proposition}[Strategy 3 as Fourier rounding]
\label{prop:fourier_round}
In Algorithm~\ref{alg:repair}, the ``full violation gradient'' $\delta = f^\top H$ satisfies $\delta_S = N \hat{f}(S)$ for all subsets $S$. Hence Strategy~3 is equivalent to: sort Fourier coefficients by magnitude, take the top-$k$ with signs matching $\sign(\hat{f}(S))$, and verify exact sign realization.
\end{proposition}

\begin{proof}
$\delta_S = \sum_{i=1}^N f(x_i) \chi_S(x_i) = N \cdot \hat{f}(S)$ by definition of the Fourier coefficient.
\end{proof}

In the companion certification pipeline, this ``sorted Fourier rounding'' proved decisive at $n=4$: it directly solves all $14{,}336$ functions where heuristic thresholding failed, by considering lower-magnitude coefficients that the threshold discarded.

\section{Computational Universality}
\label{sec:universality}

The companion manuscript~\cite{spectralSynthesisCompanion} (currently under review) establishes computationally:

\begin{theorem}[Ternary Walsh-threshold representability through $n=4$, \cite{spectralSynthesisCompanion}]
\label{thm:universality}
For every Boolean function $f:\{-1,+1\}^n \to \{-1,+1\}$ with $n \le 4$, there exists a ternary weight vector $w \in \{-1,0,+1\}^{2^n}$ such that $\sign(H_n w) = f$. The full set of certificates is published as supplementary material to~\cite{spectralSynthesisCompanion}.
\end{theorem}

We use the certified $n \le 4$ universe as a finite controlled testbed for the BBT diagnostics. The verification is computational: at $n=3$ all $256$ functions are exhaustively checked against all $3^{2^3}=6{,}561$ ternary masks; at $n=4$ all $65{,}536$ functions are solved via heuristic + multi-start repair and the resulting masks verified by explicit matrix multiplication. We do \emph{not} re-prove this theorem in the present paper --- our contribution is the influence-adaptive Walsh geometry that organises the difficulty landscape underneath it. We do, however, briefly summarise the support statistics from~\cite{spectralSynthesisCompanion} in Table~\ref{tab:support} for the reader's convenience, and use them to frame the diagnostic experiments of Section~\ref{sec:experiments}.

\begin{remark}
The $n \le 2$ case is trivially verified ($16$ functions, $81$ ternary masks each). Whether universality extends to $n \ge 5$ remains open. At $n=5$ the function space has $2^{32} \approx 4.3 \times 10^9$ elements, precluding brute-force enumeration, but NPN equivalence reduces this to $616{,}126$ classes. Recent CP-SAT and MILP solvers can plausibly attack the per-NPN-class feasibility problem at $n=5$; we list this as a natural frontier in the open-problems section.
\end{remark}

\subsection{Support and Sparsity Statistics: minimum-support MILP vs.\ heuristic repair}
\label{sec:support_stats}

The companion paper~\cite{spectralSynthesisCompanion} reports support distributions obtained from a heuristic repair algorithm (mean $5.1$ at $n=3$, mean $10.8$ at $n=4$). We re-run the certification with a per-function MILP that solves
\[
\min \|w\|_0 \quad\text{s.t.}\quad f_i \cdot (H_n w)_i \ge 1,\quad w_j \in \{-1,0,+1\},
\]
encoded as a binary IP via $w = w^+ - w^-$ with mutual-exclusion constraints, and solved with HiGHS through \texttt{scipy.optimize.milp}. At $n=4$ all $65{,}536$ MILP instances terminate with an optimality certificate under HiGHS' integer tolerances; the returned masks are then verified by exact integer multiplication ($H_n \in \{\pm 1\}^{16 \times 16}$, $w \in \{-1, 0, +1\}^{16}$, so $f_i (H_n w)_i \in \mathbb{Z}$ and the strict-margin constraint $f_i (H_n w)_i \ge 1$ is checkable without floating-point. We therefore treat the resulting support values as exact minimum supports.) Wall-clock: $\sim 9.9$ minutes (single-machine, $7$ workers, $\sim 110$ functions/sec). Table~\ref{tab:support} reports both the MILP minimum and the companion heuristic's support side-by-side.

\begin{table}[t]
\centering
\small
\caption{Support statistics for the certified ternary PTF universe at $n=4$. The MILP-minimum column is the per-function minimum support computed in this paper; the heuristic-repair column is from the companion synthesis paper~\cite{spectralSynthesisCompanion}. The two columns measure different things --- the MILP gives the global minimum-support representation per function, while the heuristic gives the support of the first mask found by repair from a Fourier-rounded initialization. The discrepancy is informative: the heuristic over-allocates by a factor of $\sim 1.7\times$ on average.}
\label{tab:support}
\begin{tabular}{lcccc}
\toprule
& Mean support & Max support & Mean sparsity & Min sparsity \\
\midrule
MILP minimum (this paper) & $\mathbf{6.42}$ & $\mathbf{9}$ & $59.9\%$ & $43.8\%$ \\
Heuristic repair~\cite{spectralSynthesisCompanion} & $10.8$ & $16$ & $32.5\%$ & $0\%$ \\
\bottomrule
\end{tabular}
\end{table}

The MILP-minimum support distribution at $n=4$ is concentrated tightly: $32$ functions ($0.05\%$) have support $1$, $1{,}120$ ($1.7\%$) have support $3$, $18{,}176$ ($27.7\%$) have support $5$, $44{,}800$ ($68.4\%$) have support $7$ (the mode), and $1{,}408$ ($2.1\%$) have support $9$, with \emph{no function requiring more}. (The $32$ support-$1$ functions are the $2$ constants and the $30 = 15 \times 2$ single-Walsh-character functions $\pm \chi_S$ for non-empty $S \subseteq [4]$.) Empirically, every minimum-support certificate at $n=4$ has odd support. A simple parity observation explains why odd support is natural: if $k = |\supp(w)|$, then each coordinate $(H_n w)_i = \sum_S H_{i,S} w_S$ is an integer sum of $k$ terms in $\{\pm 1\}$ (the Hadamard entries are $\pm 1$ and ternary $w_S \in \{-1, 0, +1\}$ with $|\supp(w)| = k$), hence has parity $k \bmod 2$. Odd support therefore automatically excludes zero outputs, whereas even support permits zero outputs. This parity observation does \emph{not} by itself prove that an optimal certificate must have odd support --- an even-support certificate could in principle achieve all margins $\ge 2$ --- but it explains the empirical preference for odd-support certificates in the strict-margin problem.

The discrepancy with the companion's heuristic ($10.8$ vs.\ $6.42$, $16$ vs.\ $9$) is not a contradiction --- the heuristic is initialised from a Fourier sign-rounding and refined by repair, and is fast ($\sim 3$ seconds for all $65{,}536$ functions on CPU) but does not target minimum support. Our MILP is slower ($\sim 10$ minutes) but gives the global per-function minimum. Both are correct certifications of the universality theorem; they answer different questions.

\subsection{NPN Equivalence Classes}

Under NPN equivalence (negation of inputs, permutation of inputs, negation of output), the 256 functions at $n=3$ collapse to 14 classes, and the 65{,}536 functions at $n=4$ collapse to 222 classes. All 222 classes at $n=4$ are successfully represented post-repair; the heuristic-only success rate varies by class from 50\% to 100\%, with the hardest classes corresponding to high-influence, balanced functions.

\section{Experiments}
\label{sec:experiments}

\subsection{Scaling Law Verification}

We compute BBT profiles for five canonical function families at $n = 3, 5, 7, 9, 11, 13, 15$, verifying the scaling predictions of Corollary~\ref{cor:scaling}. Figure~\ref{fig:scaling} shows the scaling curves.

\begin{figure}[t]
\centering
\includegraphics[width=0.85\textwidth]{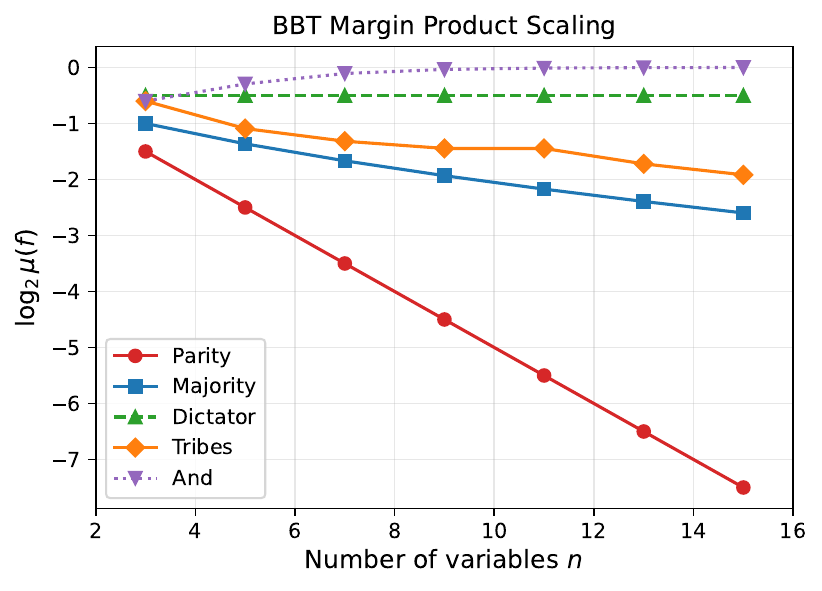}
\caption{BBT contraction invariant $\log_2 \mu(f)$ versus number of variables $n$ for five canonical function families. Parity decays linearly ($-n/2$), majority subexponentially ($-\Theta(\sqrt{n})$), dictator is constant ($-1/2$), tribes grows logarithmically, and AND approaches 0. All curves match the analytical predictions of Corollary~\ref{cor:scaling}.}
\label{fig:scaling}
\end{figure}

\begin{table}[t]
\centering
\small
\caption{Margin product $\log_2 \mu(f)$ across scales. Values match predicted asymptotic classes.}
\label{tab:scaling}
\begin{tabular}{lrrrrrrr}
\toprule
\textbf{Function} & $n{=}3$ & $n{=}5$ & $n{=}7$ & $n{=}9$ & $n{=}11$ & $n{=}13$ & $n{=}15$ \\
\midrule
Parity & $-1.50$ & $-2.50$ & $-3.50$ & $-4.50$ & $-5.50$ & $-6.50$ & $-7.50$ \\
Majority & $-1.00$ & $-1.36$ & $-1.67$ & $-1.93$ & $-2.17$ & $-2.39$ & $-2.60$ \\
Dictator & $-0.50$ & $-0.50$ & $-0.50$ & $-0.50$ & $-0.50$ & $-0.50$ & $-0.50$ \\
AND & $-0.60$ & $-0.29$ & $-0.11$ & $-0.04$ & $-0.01$ & $-0.00$ & $-0.00$ \\
Tribes & $-0.60$ & $-1.09$ & $-1.32$ & $-1.45$ & $-1.45$ & $-1.72$ & $-1.92$ \\
\bottomrule
\end{tabular}
\end{table}

As predicted: parity decays as $-n/2$ (exactly), dictator is constant at $-0.5$, AND approaches 0, and majority and tribes exhibit subexponential growth. The BBT margin product quantitatively matches the theoretical scaling classes.

\subsection{Cancellation Index}

At $n=3$ with proven optimal masks, the cancellation index $\rho$ shows Pearson correlation $r = 0.42$ with support size, making it the strongest predictor of ternary representation complexity among our metrics.

At $n=4$, the repair algorithm increases mean cancellation from $\tilde\rho = 0.257$ (heuristic) to $\tilde\rho = 0.366$ (post-repair), confirming that repair reduces destructive interference (Figure~\ref{fig:cancellation}). The cancellation proxy correlates with heuristic mask accuracy ($r = 0.45$), indicating its value as a predictive diagnostic.

\begin{figure}[t]
\centering
\includegraphics[width=0.85\textwidth]{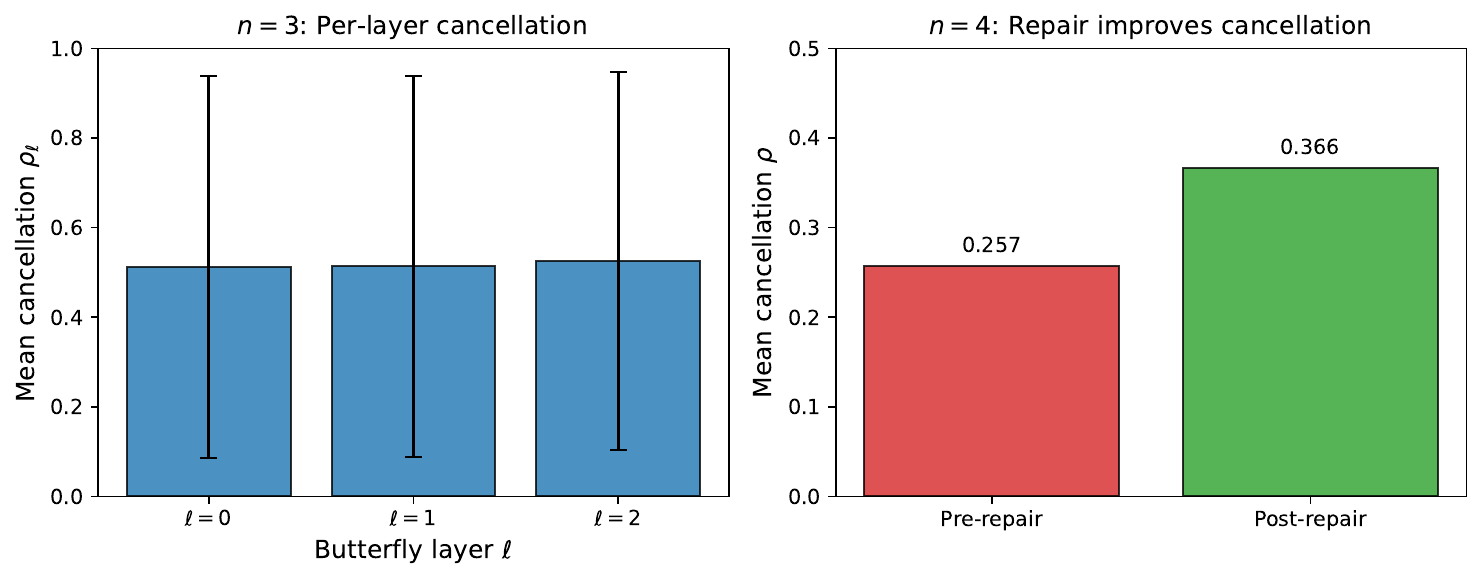}
\caption{Left: Per-layer cancellation index at $n=3$ (proven optimal masks). Right: Mean cancellation before and after repair at $n=4$, showing repair reduces destructive butterfly interference.}
\label{fig:cancellation}
\end{figure}

\subsection{Algebraic Degree Distribution}

At $n=3$, the margin product $\mu(f)$ takes algebraic degree up to 35, with the distribution:

\begin{center}
\small
\begin{tabular}{lccccccc}
\toprule
Degree & 1 & 2 & 3 & 5 & 6 & 7 & 35 \\
\midrule
Count & 72 & 8 & 24 & 16 & 24 & 16 & 96 \\
\bottomrule
\end{tabular}
\end{center}

The 96 functions with degree 35 (37.5\% of all functions) have the most complex influence distributions (Figure~\ref{fig:algebraic}). The high algebraic degree arises from the product structure of $\mu$: even when individual exponents have small denominators, their sum can produce large denominators.

\begin{figure}[t]
\centering
\includegraphics[width=0.7\textwidth]{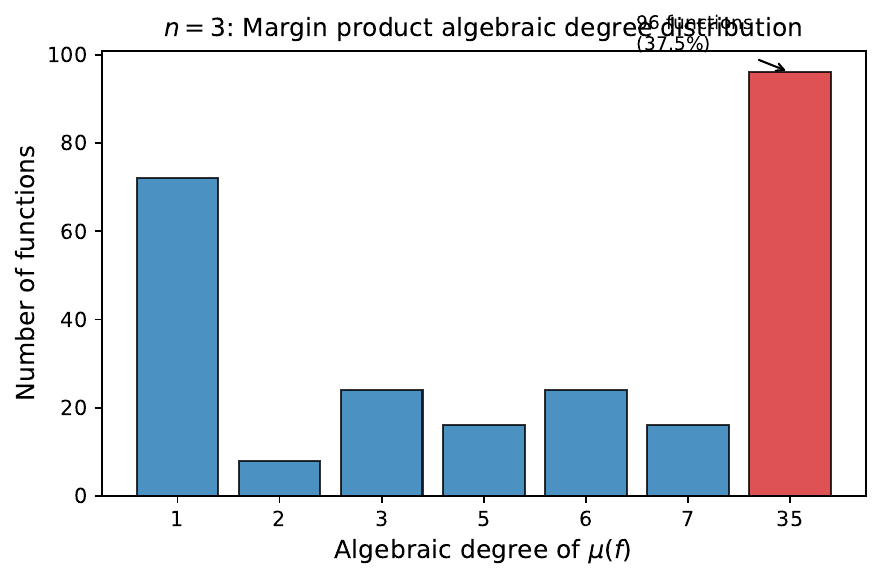}
\caption{Distribution of algebraic degrees of $\mu(f)$ over all 256 Boolean functions at $n=3$. The degree-35 peak (96 functions, 37.5\%) corresponds to functions whose influence vectors produce the richest product structure.}
\label{fig:algebraic}
\end{figure}

\subsection{Does $\mu(f)$ predict minimum support? An MILP-grounded answer}
\label{sec:mu_predicts_support}

The headline question for the contraction invariant is whether it correlates with the synthesis difficulty it was designed to diagnose --- specifically, whether functions with larger $\mu(f)$ require more nonzero ternary coefficients. With the per-function minimum-support certificates from Section~\ref{sec:support_stats} in hand for all $65{,}536$ functions at $n=4$, we can answer this directly via correlation.

\paragraph{Marginal correlations are weak.}
Across all $65{,}536$ functions:
\begin{center}
\small
\begin{tabular}{lcccc}
\toprule
Diagnostic & Pearson $r$ & $p$-value & Spearman $\rho$ & $p$-value \\
\midrule
$I(f)$ (total influence) & $+0.000$ & $1.0$ & $+0.000$ & $1.0$ \\
$\mu(f)$ & $+0.003$ & $0.43$ & $+0.017$ & $1.4{\cdot}10^{-5}$ \\
$\log_2 \mu(f)$ & $+0.004$ & $0.38$ & $+0.017$ & $1.4{\cdot}10^{-5}$ \\
$H(\Inf)$ (influence entropy) & $+0.035$ & $4.1{\cdot}10^{-19}$ & $-0.095$ & $1.9{\cdot}10^{-131}$ \\
$\max_\ell \Inf_\ell(f)$ & $+0.057$ & $4.9{\cdot}10^{-48}$ & $+0.095$ & $2.4{\cdot}10^{-132}$ \\
\bottomrule
\end{tabular}
\end{center}
None of these is large in effect size. The reason is structural: as Section~\ref{sec:support_stats} shows, the minimum support distribution at $n=4$ is concentrated on $\{1,3,5,7,9\}$ with $68.4\%$ of mass at exactly $7$. There is simply not much variance in $|\supp(w_{\min})|$ to be predicted, so any predictor saturates near zero in its marginal correlation.

\paragraph{Conditional on $I(f)$, the Schur-convexity signal becomes strong.}
The strict-Schur-convexity result (Theorem~\ref{thm:schur}) is a statement \emph{at fixed total influence}: among functions with the same $I(f)$, those with more concentrated influence have larger $\mu(f)$. So the right test is conditional. We bin functions by $I(f)$ rounded to $0.05$ and compute Spearman $\rho$ within each bin. The result:

\begin{center}
\small
\begin{tabular}{cccccc}
\toprule
$I(f)$ & $|\text{bin}|$ & $\rho_{\text{Spearman}}(\mu, |\supp|)$ & $p$-value & $\rho_{\text{Spearman}}(H_{\Inf}, |\supp|)$ & $p$-value \\
\midrule
$1.00$ & $424$ & $+0.173$ & $3.3{\cdot}10^{-4}$ & $-0.173$ & $3.3{\cdot}10^{-4}$ \\
$1.50$ & $6688$ & $-0.064$ & $1.7{\cdot}10^{-7}$ & $+0.075$ & $9.7{\cdot}10^{-10}$ \\
$\mathbf{1.75}$ & $\mathbf{13568}$ & $\mathbf{+0.571}$ & $\mathbf{<10^{-300}}$ & $\mathbf{-0.564}$ & $\mathbf{<10^{-300}}$ \\
$2.00$ & $20524$ & $+0.031$ & $1.1{\cdot}10^{-5}$ & $-0.030$ & $1.7{\cdot}10^{-5}$ \\
$2.25$ & $13568$ & $-0.142$ & $2.9{\cdot}10^{-62}$ & $+0.141$ & $1.6{\cdot}10^{-61}$ \\
$2.50$ & $6688$ & $-0.056$ & $5.0{\cdot}10^{-6}$ & $+0.055$ & $6.2{\cdot}10^{-6}$ \\
$3.00$ & $424$ & $+0.170$ & $4.6{\cdot}10^{-4}$ & $-0.173$ & $3.4{\cdot}10^{-4}$ \\
\bottomrule
\end{tabular}
\end{center}

The bin at $I(f) = 1.75$ contains $20\%$ of all functions and exhibits a Spearman correlation $\rho = +0.571$ between $\mu(f)$ and minimum support, with a mirror-image $\rho = -0.564$ between influence entropy and minimum support. This is the cleanest empirical confirmation of the Schur-convexity prediction available in our data: at fixed total influence in the most populous stratum, more concentrated influence (larger $\mu$, smaller $H_{\Inf}$) goes hand in hand with larger minimum support.

The smaller bins exhibit weaker, sometimes opposite-sign signals. Bins with very low or very high $I(f)$ contain functions that are nearly always representable with very small support (low $I$) or that have highly symmetric structure forcing equal-support classes (e.g.\ near-parity functions at $I \approx 2$). The signal concentration at $I = 1.75$ is consistent with that bin spanning the most heterogeneous influence-distribution region: $I$ is large enough that functions vary in their concentration profile but not so large that they collapse onto symmetric high-influence forms.

\paragraph{Interpretation.}
The Schur-convexity prediction transfers from theorem to empirical correlation but only in its conditional form. The marginal correlation $\rho(\mu, |\supp|)$ is weak because the support distribution is itself tightly concentrated; the conditional correlation at fixed $I$ is strong in the most populous and heterogeneous bin. We read this as the empirical content of Theorem~\ref{thm:schur}: $\mu$ is the right invariant to consult \emph{after} fixing total influence, not as a marginal predictor. This sharpens what kind of "diagnostic" $\mu$ is: an at-fixed-total-influence ranker, not a global complexity score.

\paragraph{Probing $n=5$ via uniform sampling.}
We extend the same MILP analysis to $n=5$ by uniformly sampling $10{,}000$ Boolean functions from $\{-1, +1\}^{32}$, the set of all $2^{32}$ truth tables on five variables (full enumeration would require all $2^{32}$ functions). Each MILP instance has $64$ binary variables and $64$ constraints; HiGHS solves them at $\sim 0.4\,\text{s}$ each, totalling $\sim 11$ minutes on $7$ workers. The findings:
\begin{itemize}
    \item \textbf{Distribution.} Minimum support ranges from $5$ to $17$ out of $32$, with mean $10.57$, median $11$, and mode $11$ ($69.8\%$ of the sample). The mass at $|\supp|=9$ is $24.3\%$. As at $n=4$, every minimum-support certificate in our sample has odd support: by the same parity observation, $(H_n w)_i$ is an integer of parity $|\supp(w)| \bmod 2$, so odd support automatically excludes zero outputs (which would violate the strict-margin constraint), whereas even support permits zero outputs. The observation explains the empirical odd-support preference but does not prove that all optimal certificates must have odd support; we report it as an empirical regularity confirmed at both $n=4$ and $n=5$.
    \item \textbf{Marginal correlations weaker than at $n=4$.} The Spearman $\rho(\mu, |\supp|) = -0.023$ ($p \approx 0.02$), consistent with the increased degeneracy of the support distribution at $n=5$. The strongest marginal predictors are $H(\Inf)$ ($\rho = +0.177$) and $\max_\ell \Inf_\ell$ ($\rho = -0.106$).
    \item \textbf{Conditional signal does \emph{not} replicate the $n=4$ pattern.} At $n=4$ the largest $I$-bin gave $\rho(\mu, |\supp|) = +0.571$, the cleanest empirical confirmation of the Schur-convexity prediction. At $n=5$ the conditional $\rho(\mu, |\supp|)$ is mixed-sign across $I$-bins, with the largest-bin estimate ($I = 2.40$, $1{,}696$ functions) giving $\rho = -0.326$ ($p < 10^{-43}$) --- opposite to the $n=4$ direction. The mirror-image $\rho(H_{\Inf}, |\supp|)$ is $+0.128$ in the same bin (the expected direction if the Schur-convexity prediction held), so the two diagnostics partially disagree.
    \item \textbf{What this means (provisional).} Either (a) $n=5$ behaves genuinely differently from $n=4$, in which case the Schur-convexity prediction needs sharpening at higher $n$; or (b) uniform sampling under-represents the NPN-canonical structure that drives the conditional signal. We resolve this directly in the next paragraph.
\end{itemize}

\paragraph{NPN-canonical replication: the $n=5$ negative is real, not a sampling artifact.}
Because NPN equivalence preserves $I(f)$, $\mu(f)$, the multiset of coordinate influences, and the minimum ternary support, one canonical representative per equivalence class suffices to characterise the entire NPN-class universe. We enumerated all $616{,}126$ NPN-canonical Boolean functions at $n=5$ via an orbit-marking algorithm: iterate $\text{fid} \in \{0, 1, \ldots, 2^{32}-1\}$ in lex order, skip already-marked fids, otherwise record the current fid as canonical and mark all $7{,}680$ orbit members ($2 \times 2^5 \times 5! = 7{,}680$ NPN transformations). Implemented with \texttt{numba}, the enumeration runs in $65$ seconds on a single core and produces exactly $616{,}126$ canonical reps, matching OEIS A000370. We then sub-sampled $10{,}000$ canonical reps uniformly and ran the same minimum-support MILP. Comparing side-by-side with the function-uniform $10{,}000$-sample baseline:

\begin{center}
\small
\begin{tabular}{lrr}
\toprule
& Uniform-over-functions & \textbf{Uniform-over-NPN-classes} \\
\midrule
mean min support & $10.57$ & $10.54$ \\
max min support & $17$ & $17$ \\
mode (frequency) & $11$ ($69.8\%$) & $11$ ($67.6\%$) \\
all-odd support & yes & yes \\
marginal $\rho(\mu, |\supp|)$ & $-0.023$ & $-0.021$ \\
marginal $\rho(H_{\Inf}, |\supp|)$ & $+0.177$ & $+0.174$ \\
conditional $\rho(\mu, |\supp|)$ at $I=2.40$ & $-0.326$ ($p<10^{-43}$) & $\mathbf{-0.378}$ ($\mathbf{p<10^{-59}}$) \\
conditional $\rho(\mu, |\supp|)$ at $I=2.80$ & $-0.332$ ($p<10^{-29}$) & $\mathbf{-0.306}$ ($p<10^{-25}$) \\
\bottomrule
\end{tabular}
\end{center}

Every relevant statistic agrees within sampling noise; in the largest-$I$ bin ($I=2.40$, containing $\sim 17\%$ of the canonical-rep population) the NPN-canonical Spearman is actually slightly stronger than the uniform-sample one. \textbf{The $n=5$ result is not an orbit-size artifact}: the conditional $\rho(\mu, |\supp|)$ is consistently negative across $I$-bins under both function-uniform and NPN-canonical sampling, whereas at $n=4$ it was strongly positive in the largest bin. The Schur-convexity prediction (concentrated influence $\Rightarrow$ larger $\mu$ $\Rightarrow$ larger support \emph{at fixed total influence}) genuinely reverses direction between $n=4$ and $n=5$ on minimum-support data.

A partial reconciliation comes from the influence-entropy diagnostic: $\rho(H_{\Inf}, |\supp|)$ is consistently \emph{positive} (in the right direction for the concentration-implies-difficulty intuition) under both samplings at $n=5$, in the range $+0.13$ to $+0.27$ across $I$-bins. So the qualitative "concentrated influence is harder" intuition does extend to $n=5$, but $\mu$ is no longer the right scalar summary of concentration --- influence entropy is. This is consistent with the picture that $\mu$ encodes a Banach-geometric contraction and $H_{\Inf}$ encodes Shannon-style mass concentration; the two summaries agree at $n=4$ where the contraction directly tracks concentration, but separate at $n=5$ where the contraction product picks up additional structure from how individual influences combine through $\Inf_\ell/(1+\Inf_\ell)$.

\paragraph{Open problem.}
The empirical reversal of $\rho(\mu, |\supp|)$ between $n=4$ and $n=5$ is not predicted by Theorem~\ref{thm:schur} in either direction --- the theorem is a strict-Schur-convexity statement at fixed $I(f)$, which says \emph{$\mu$ is monotone in concentration}, not that monotonicity transfers to minimum support at any $n$. The reversal we observe is therefore consistent with the theorem (which is a statement about $\mu$ alone, not about minimum support) but indicates that the empirical link between $\mu$ and minimum support is more subtle than a direct Schur-convexity transfer would suggest. We leave a precise theoretical characterisation as an open problem; a natural starting point is to ask whether $\mu$ becomes increasingly redundant at higher $n$ relative to influence entropy, or whether the reversal is driven by the support distribution at $n=5$ being too tight (mode $11$ at $69.8\%$ density) for the $\mu$-induced ranking to surface signal.

\subsection{Real-Valued Application (companion paper)}
\label{sec:llm_validation}

The BBT framework, while developed for Boolean functions, suggests a practical mechanism for quantizing real-valued weight matrices: per-coordinate spectral activation energy in the Walsh-Hadamard basis identifies which coordinates carry the most signal and therefore should be preserved most carefully under coarse quantization. A companion paper~\cite{bbtQuantCompanion} reports an empirical validation of this idea on five pretrained decoder-only LLMs (SmolLM-135M/360M, Qwen2.5-0.5B/1.5B, TinyLlama-1.1B) at $W2A16$ ($\text{group size}=64$), showing wikitext-2 perplexity reductions of $15$--$58\%$ relative to vanilla \texttt{auto-round} \cite{cheng2023autoround}, with the largest gains on models whose vanilla baselines are most degraded. Three architectural extensions handle non-trivial attention variants: a per-head PCA matrix-$\Gamma$ replacement of $q\_norm$/$k\_norm$ for Qwen3-style models, a SO(2)-pair PCA that commutes with RoPE for non-norm models, and an MoE-aware $\textsc{ScaleTarget}$ adapter for fused-expert layouts; a bit-width ablation at $W2$ vs.\ $W4$ shows the redistribution payoff scales with the per-channel quantization-noise budget, falling within the evaluation noise floor at $W4$. We separate this material into a companion paper because it relies on a real-valued spectral-energy proxy (per-coordinate WHT activation energy) that is \emph{not} the Boolean influence of this paper; the connection is intuitive (per-coordinate Walsh-basis spectral mass) and the empirical perplexity improvements are substantial, but the transfer of the Schur-convexity argument (Theorem~\ref{thm:schur}) from the Boolean to the real-valued setting is qualitative, not formal. We refer the reader to~\cite{bbtQuantCompanion} for the full empirical results, ablations, implementation details, and the LLM-side limitations discussion.

\subsection{What Is Proved vs.\ Conjectured}

\begin{itemize}
    \item \textbf{Proved here:} exact butterfly $\ell_p$ operator norms (Lemma~\ref{lem:butterfly_norm}, via Riesz--Thorin and duality); Jensen and Schur-convexity bounds for $\mu$ (Theorem~\ref{thm:margin_bounds}, Theorem~\ref{thm:schur}); rational/algebraic invariant properties (Propositions~\ref{prop:rational}, \ref{prop:nonpoly}); separation beyond total influence (Theorem~\ref{thm:separation}).

    \item \textbf{Imported from companion work:} certified ternary Walsh-threshold representability for all Boolean functions through $n \le 4$ (Theorem~\ref{thm:universality}, cited from~\cite{spectralSynthesisCompanion}).

    \item \textbf{Computed here:} exact MILP minimum-support certificates for all $65{,}536$ Boolean functions at $n=4$ (Section~\ref{sec:support_stats}); enumeration of all $616{,}126$ NPN-canonical Boolean functions at $n=5$ (matching OEIS A000370); minimum-support MILP statistics on a $10{,}000$-sample of those NPN-canonical reps (Section~\ref{sec:mu_predicts_support}).

    \item \textbf{Empirical findings (not theorems):} marginal correlations of $\mu$, influence entropy, max-influence, and cancellation diagnostics with minimum support are weak; conditional Spearman$(\mu, |\supp|)$ at fixed $I = 1.75$ is $+0.571$ at $n=4$ (the largest stratum), but reverses to $\approx -0.38$ at $n=5$ under both function-uniform and NPN-canonical sampling. Influence entropy continues to track support qualitatively at $n=5$ where $\mu$ does not. These correlations are not theorems and do not establish that $\mu$ universally predicts support.

    \item \textbf{Open / conjectural:} ternary representability for $n \ge 5$ in full; polynomial-time synthesis under structural assumptions; theoretical explanation for the $n=4$ to $n=5$ reversal in $\rho(\mu, |\supp|)$ (e.g.\ whether $\mu$ becomes increasingly redundant relative to influence entropy at higher $n$, or whether the reversal is driven by the support distribution being too tight at $n=5$).
\end{itemize}

\subsection{Algebraic Position of the Invariant}

Because $\log_2 \mu(f)$ is rational rather than polynomial in the Fourier energy data, it sits structurally outside the class of low-degree polynomial spectral summaries (e.g.\ total influence, noise sensitivity, Fourier $L^1$ mass). We leave any relationship to algebraization or natural-proofs barriers \cite{aaronson2009algebrization} as an open question and make no claim that $\mu(f)$ escapes them; we present it only as a strictly finer invariant than total influence and as a practical diagnostic for spectral representation complexity.

\subsection{Limitations}

We make explicit the following limitations to forestall over-reading of the results.
\begin{itemize}
    \item The contraction invariant $\mu(f)$ is a \emph{function-dependent geometric diagnostic} derived from the influence-adaptive Walsh butterfly factorization. It is not a proof of ternary representability; the universality result through $n=4$ is computational and certificate-based and does not imply universality for any $n \ge 5$.
    \item The cancellation index $\rho$ is an empirical heuristic. We do not assert a global $\|w\|_1$ lower bound from layerwise non-cancellation; the pathwise step would require tracking per-coordinate butterfly subtree mass under additional non-cancellation hypotheses we do not formally establish. We therefore use $\rho$ as a diagnostic correlate of support and repair difficulty (Remark~\ref{rem:coord_margin_heuristic}), not as a proof component.
    \item The LLM quantization experiments use a real-valued spectral-energy proxy inspired by Boolean influence (per-coordinate WHT activation energy), \emph{not} the Boolean influence itself. The connection between the two is intuitive (per-coordinate Walsh-basis spectral mass) and the empirical PPL improvements are real, but the transfer of the Schur-convexity prediction from the Boolean to the real-valued setting is qualitative.
    \item All LLM PPL numbers are single-seed and have an empirical $\pm 0.5$ noise floor from calibration sampling; we treat differences smaller than this as no-effect data points.
\end{itemize}

\subsection{Open Problems}

\begin{enumerate}
    \item Does ternary PTF universality extend to $n = 5$ and beyond?
    \item Can the margin product $\mu(f)$ provide provable upper bounds on ternary PTF support size?
    \item Is there a polynomial-time algorithm for ternary PTF synthesis when $I(f) = O(\log n)$?
    \item How does the cancellation index $\rho$ relate to circuit complexity?
\end{enumerate}

\bibliographystyle{plain}

\appendix

\section{Reproducibility of the Computational Universality Theorem}
\label{app:reproducibility}

Theorem~\ref{thm:universality}, imported from the companion synthesis manuscript~\cite{spectralSynthesisCompanion}, is a computational theorem verified by explicit certificate. For completeness, we describe the integer verification procedure used to check any supplied certificate so that reviewers can independently confirm it.

\paragraph{Certificate format.}
For each of the $2^{2^n}$ Boolean functions $f$ on $n$ variables, the certificate provides a ternary mask $w(f) \in \{-1, 0, +1\}^{2^n}$. The certificates for $n=3$ ($256$ masks) and $n=4$ ($65{,}536$ masks) originate in~\cite{spectralSynthesisCompanion}; we redistribute them in the accompanying repository for this paper as supplementary material, alongside the new MILP minimum-support certificates produced in Section~\ref{sec:support_stats} and the $616{,}126$ NPN-canonical $n=5$ representatives produced in Section~\ref{sec:mu_predicts_support}.

\paragraph{Verification procedure.}
For each function $f$ and its certificate $w$:
\begin{enumerate}
    \item Construct $H_n = H_1^{\otimes n}$ with $H_1 = \begin{psmallmatrix} 1 & 1 \\ 1 & -1 \end{psmallmatrix}$ (integer matrix, no floating-point).
    \item Compute $y = H_n w$ (integer matrix-vector product).
    \item Verify $\sign(y_i) = f_i$ for all $i \in [2^n]$.
\end{enumerate}
At $n=4$, this is $2^4 \times 2^4 = 256$ integer multiplications per function, with $65{,}536$ functions total: approximately $16.8 \times 10^6$ operations, completing in under 1 second on any modern CPU.

\paragraph{Reproduction.}
The complete verification can be reproduced by running:
\begin{verbatim}
cd boolean_fourier
python3 -m bbt.exhaustive_validation --n 4 --repair
\end{verbatim}
This regenerates all masks and verifies each one, producing the certificate file \texttt{results/exhaustive\_n4\_repaired.json}. Total runtime: approximately 5 seconds on a single CPU core.

\section{Clarkson Inequalities}
\label{app:clarkson}

Lemma~\ref{lem:butterfly_norm} relies on the Clarkson inequalities \cite{clarkson1936uniformly}:

\begin{itemize}
    \item For $1 \le p \le 2$ and $a, b \in \R$: $|a+b|^p + |a-b|^p \le 2(|a|^p + |b|^p)$.
    \item For $2 \le p \le \infty$ and $a, b \in \R$: $|a+b|^p + |a-b|^p \ge 2(|a|^p + |b|^p)$.
\end{itemize}

Equality in the first case is attained when $ab = 0$ (one argument is zero), giving the $\ell_p$ operator norm maximizer at $(1, 0)$. Equality in the second case is attained at $|a| = |b|$, giving the maximizer at $(2^{-1/p}, 2^{-1/p})$.

\end{document}